\documentclass[10pt]{article}
\usepackage{fullpage}
\usepackage{times,amsmath,amsfonts,amssymb,url,color,graphicx,natbib}

\usepackage{booktabs} 

\DeclareMathOperator*{\prob}{\mathbb{P}}

\newtheorem{theorem}{Theorem}
\newtheorem{definition}{Definition}
\newtheorem{lemma}{Lemma}

\newcommand{\BlackBox}{\rule{1.5ex}{1.5ex}}  
\newenvironment{proof}{\par\noindent{\bf Proof\ }}{\hfill\BlackBox\\[2mm]}

\newcommand{\lemref}[1]{Lemma~\ref{#1}}

\newcommand{\secref}[1]{Section~\ref{#1}}

\newif\ifcomments
\commentstrue
\ifcomments

\usepackage[utf8]{inputenc}
\usepackage{authblk} 
\usepackage{xcolor} 

\author[1]{Shai Shalev-Shwartz}
\author[1]{Amnon Shashua}
\author[2]{Gal Beniamini}
\author[3]{Yoav Levine}
\author[4]{Or Sharir}
\author[4]{Noam Wies}
\author[4]{Ido Ben-Shaul}
\author[4]{Tomer Nussbaum}
\author[5]{Shir Granot Peled}

\affil[1]{\small Conceptualization, Formal Analysis, Writing}
\affil[2]{Conceptualization, Methodology, Investigation, Software}
\affil[3]{Writing, Methodology, Supervision}
\affil[4]{Methodology, Investigation, Software}
\affil[5]{Software}

\date{AAI, December 2024} 

\begin{document}

\title{Artificial Expert Intelligence through PAC-reasoning}


\maketitle

\textit{The Jewish Talmud asks which type of scholar is preferable:} 

\textit{-- ``Sinai Desert", a scholar who excels in broad knowledge across many subjects, or} 

\textit{-- ``Uprooter of Mountains", a scholar proficient in one subject with exceptional intellectual sharpness? 
}

\begin{abstract}
Artificial Expert Intelligence (AEI) seeks to transcend the limitations of both Artificial General Intelligence (AGI) and narrow AI by integrating domain-specific expertise with critical, precise reasoning capabilities akin to those of top human experts. Existing AI systems often excel at predefined tasks but struggle with adaptability and precision in novel problem-solving. To overcome this, AEI introduces a framework for ``Probably Approximately Correct (PAC) Reasoning". This paradigm provides robust theoretical guarantees for reliably decomposing complex problems, with a practical mechanism for controlling reasoning precision. In reference to the division of human thought into System 1 for intuitive thinking and System 2 for reflective reasoning~\citep{tversky1974judgment}, we refer to this new type of reasoning as System 3 for precise reasoning, inspired by the rigor of the scientific method. AEI thus establishes a foundation for error-bounded, inference-time learning. 

\end{abstract}

\section{Introduction}

What does it take to be a great scientist or domain expert? Beyond technical skills and mastery of domain-specific tools---where AI already excels---there is an elusive quality that distinguishes leading human experts: actual \textit{intelligence}, as manifested by the ability to innovatively synthesize knowledge, while at the same time think critically in the sense of separating correct from incorrect statements and acknowledging the boundaries of knowledge. As a result, great scientists and experts advance humanity by solving novel problems. This paper introduces Artificial Expert Intelligence (AEI) as a new paradigm that combines current AI’s knowledge and skills with the intelligence characteristic of top human experts. We propose a concrete, constructive definition of AEI which provides a blueprint for building AEI systems. 

On our path to understanding the intelligence aspect exhibited by human experts, let us consider the ongoing discussion on the definition of intelligence. 
\citet*{legg2007collection} compiled a list of 70 informal definitions of intelligence.  In his famous imitation game, \citet{turing} tackled the ``can machines think'' question by their ability to be indistinguishable from a human during a conversation. This ability is a specific skill, and therefore does not seem to convey the full generality of what intelligence is. 
\citet{minsky1988society} defines intelligence as the ability to solve the hard problems that intelligent humans would solve, extending the scope of skills, but lacks reference to adaptivity to new scenarios. 
Indeed, \cite{schank1991s} defines intelligence as getting better over time, without specifying how the system gets better over time. This allows naive forms of improvement, such as memorizing facts. According to \citet{kurzweil2000age}, intelligence is the ability to use limited resources optimally – including time – to achieve goals. This was further formalized by \citet{chollet2019measure}, who provides the ability to engineer an artificial chess player as evidence to the distinction between ``skill'' and ``intelligence''. He argues that when humans acquire some skill it provides evidence for their intelligence level, whereas machines can acquire super-human skills for some particular task  (e.g. playing chess) without being intelligent, by employing the intelligence of the human who engineered the skill.  

The above tension between superhuman ``narrow" skills and intelligence is a prominent driver of AI progress in recent years.  In a sense, the term \textit{Artificial General Intelligence} (AGI) was coined after the realization that ``narrow" super-human AI does not advance humanity. However, in eagerness to pursue the ``generality" within AGI, a fundamental question was overlooked: is generality essential for AI to solve the problems that matter most to humanity? Observably, humanity's greatest achievements often come from individuals with profound expertise in highly focused domains. Thus, we argue that generality is neither a necessary nor a sufficient condition for intelligence. 

\citet{chollet2019measure} gives a different definition of intelligence as a measure of skill-acquisition efficiency. 
In the language of learning theory (see for example~\citet{shalev2014understanding}), intelligence should be measured by the resource complexity (time and samples) of learning new task-specific concepts, given the same prior knowledge on the problem. Importantly, for a system to be intelligent, skill acquisition via learning must occur during the deployment (or inference) of the AI system, and not only during training.  Viewed in this prism, in-context learning is a manifestation of intelligence because learning occurs during inference time. In practice, however, in-context learning proved to be rather limited and, in particular, falls short in its ability to solve novel complex tasks.

Following ideas of Chain-of-Thoughts \citep{CoT} and Tree-of-Thoughts \citep{ToT}, flagship AI systems  (\textit{e.g.}, DeepSeek-R1, OpenAI-o1, QwQ~\citep{qwen2.5}) acknowledged the need for investing test-time compute during inference to improve the ability to solve novel complex tasks via a chain of sequential \emph{reasoning} steps. Reasoning can be defined as consisting of an \emph{actor} that generates candidates for the next claim, a \emph{critic} for scoring those candidates, and a \emph{search} heuristic for expanding the possibilities for the next step.
In these systems, the actor, critic, and search heuristic are all defined at training time, leaving test-time-compute solely focused on search over reasoning chains rather than on learning. In this regard, and using Chollet's definition of Intelligence, while these systems have a much more general skill, they lack intelligence in the same manner as their narrow counterpart AI systems developed to play games (e.g., AlphaZero \cite{silver2017mastering}) -- both do not learn at inference time. 

In fact, the reasoning capability of the ``general" reasoning AI engines is not as impressive as the reasoning capability of narrow AI systems for playing games (the latter surpassing top human players) because of \emph{precision}, the ability to separate correct statements from incorrect ones. Indeed, when playing games, the rules of the game dictate the valid next steps. In contrast, in general reasoning, one may argue a claim which is simply incorrect thereby collapsing the whole reasoning chain. Precision is thus a necessary condition for expert intelligence. 

Reasoners can be categorized into three types depending on the precision of each reasoning step. At one end of the spectrum are logical reasoners, where each claim is a logical consequence of previous claims, and hence the precision of each claim is guaranteed. Logical reasoning is only relevant in narrow domains,  such as theorem proving and games with clearly defined rules. At the other end of the spectrum are ``intuitive" reasoners. In intuitive reasoning, the probability that a claim is incorrect, denoted by $\epsilon$, is \emph{fixed} at inference time. As we formally show later,  this implies an exponential decay of the system precision as a function of the number of reasoning steps. This implies that an intuitive reasoner cannot deviate significantly from prior experience on complex problems since these require long reasoning chains. Consequently, an intuitive reasoner is not truly intelligent. In this paper, we advocate for a 3rd type of reasoner, inspired by the \textit{scientific method}, by gathering empirical observations that provide correctness guarantees on each reasoning step. By doing so, the reasoner can control the magnitude of $\epsilon$ during inference time by investing inference-time compute on generating sufficient examples to learn from.

Analogously to Tversky and Kahneman's division of the thought process into System 1 for instinctive thinking and System 2 for reflective reasoning~\citep{tversky1974judgment},  we identify the transition from the immediate response performed by vanilla LLMs to the actor, critic, search-based reasoning methods as a transition from System 1 to System 2. This transition empowered current AI systems, which now perform stronger reasoning than ever. We identify our proposed scientific reasoning as a distinct mode of thought, employed by scientific experts, required for expert-level intelligence: System 3 for guaranteeably precise reasoning. We acknowledge that this is not a primordial mode of thought for humans; rather, it is an invented mode of thought developed in recent centuries that is responsible for the scientific progress made by humans. This paper argues that to reach AEI, one must develop AI systems that contain System 3 thought components for guaranteeably precise reasoning.

To understand the potential impact of the guaranteeably precise reasoning component this paper advocates, an analogy to evolution of human innovation comes to mind. Intuitive reasoning was always around, and was broadly applied by humans when developing new ideas and technologies. Greek philosophers developed formal reasoning ideas, which led to beautiful and sophisticated discoveries, much deeper than discoveries made by prior intuitive reasoning methods. However, those were limited in scope due to the boundaries of the formal constructs. It was not until humans developed the rigorous, observation-based, scientific method on top of their intuition, that the current age of accelerated discovery and innovation began. We believe that Artificial Expert Intelligence, based on similar observation-based, rigorous principles, can unlock a new age of hyper-accelerated discovery and innovation.

\section{From PAC-learning to PAC-reasoning}
\label{sec:pac_reasoning}

Probably Approximate Correct (PAC) learning \citep{valiant} aims at formalizing the notion of deducing a function $g : X \to Y$ based on empirical observations (training examples) of pairs $(x,g(x))$. The model imposes two critical assumptions. First, it is assumed that there is a distribution $\mathcal{D}$ over $X$ such that we can sample independent $x$'s from $\mathcal{D}$ and receive their $g$ value as ``labels''. Secondly, that $g$ belongs to some predefined hypothesis class $H$ known to the learner as prior knowledge. The main premise of PAC-learning is that if the number of examples seen during training is sufficiently large with respect to the VC-dimension \citep{vapnik} of $H$, then we can ``probably approximately" find $g$. In particular, for finite $H$, the required number of examples (sample complexity) grows logarithmically with $|H|$.

This is a very powerful guarantee as we can choose $H$ to be the set of all functions we can implement in Python using $d$ characters and learn this class by order of $d$ examples. The ``devil", however, is the computational complexity of learning, defined by the time required to find a good candidate in $H$, which for this hypothesis class is exponential in $d$. And so, even for medium size programs with $100$ characters, the time complexity of learning is unrealistically large. Deep learning tackles the problem by defining the hypothesis class as a differentiable space with a specific structure (a neural network architecture) and searching a good function in $H$ is performed by Stochastic Gradient Descent (SGD). For some problem-structure pairs, this approach works very well --- for example, Convolutional Neural Networks (CNN) \citep{lecun1998gradient} for computer vision and Transformers \citep{vaswani2017attention} for next-word prediction. However, deep learning fails to generalize even for some simple problems, such as learning to multiply two numbers, so while we solved the computational complexity problem, we are back to a sample complexity problem. Such failures and others motivated practitioners to break the prediction task into a ``chain of thoughts" or, more generally, into a reasoning chain.

In some sense, reasoning imposes structure on the function we aim to learn at \emph{inference} time, as opposed to CNNs and Transformers, which impose structure during training time. In particular, the imposed structure is via decomposing the original problem into simpler sub-problems. Assuming that the sample complexity of each subproblem is significantly smaller than that of the original problem, we unlock the potential to significantly deviate from what we saw during training time. However, errors in the solution of each subproblem accumulate and create an overall decaying precision, which effectively puts a cap on the length of the reasoning chain. Formally, suppose that we have decomposed the original problem into $k$ subproblems and let $\epsilon_i$ be the probability that the i'th subproblem errs over random $x\sim \mathcal{D}$, then the probability that the entire reasoning chain will succeed is 
$$\prod_{i=1}^k (1-\epsilon_i)\approx \exp\left( - \sum_{i=1}^k \epsilon_i \right) $$
If the average (over the space of sub-problems) of $\epsilon_i$ is a fixed $\epsilon > 0$, then the above equals to $\exp(-k\,\epsilon)$. This would be the case for intuitive reasoners. 

In what follows, we start with formalizing the structure imposed by reasoning at inference time, and then we introduce a mechanism for controlling the magnitude of all of the $\epsilon_i$ during inference, hence enabling arbitrarily long reasoning chains. In particular, either the reasoner succeeds and can guarantee a (probably approximately) correct solution, or admits failure and responds ``I don't know", as an expert would. We refer to this framework as PAC-reasoning.

A decomposition of a problem into subproblems can be done either bottom-up or top-down. We formalize both options and underscore the merits of each approach.

\subsection{Bottom-up Reasoning} \label{sec:bottomup}

Similar to the classical PAC-learning model, we are aiming to find a function $g : X \to Y$ and  
we assume that there is a distribution $\mathcal{D}$ over $X$ such that we can sample independent $x$'s from $\mathcal{D}$.
Unlike the classical PAC model, we do not assume that $Y = \{0,1\}$ represents labels, and we do not assume that we have access to $g(x)$ for $x$'s that we sample. Furthermore, we do not assume that $g$ belongs to some predefined hypothesis class $H$. 
Instead, we make a different assumption on the structure of $g$, aiming to formalize a notion of decomposability. This structure also defines a different loss model that our reasoner receives as an alternative to the ``zero-one loss with respect to labels" that the classical PAC model employs.
Our decomposability assumption for the bottom-up reasoner relies on the concept of computation graphs.
\begin{definition}[Computation Graph]
  A computation graph is defined by $([k],E,f_1,\ldots,f_{k})$, where 
  $([k],E)$ are the sets of vertices and edges of a directed acyclic graph, with the vertices sorted topologically\footnote{That is, for every edge $(i,j) \in E$ it must hold that $i < j$.}, and with $v_1$ being the only vertex with no incoming edges. Given an input $x$, the computation graph applies the following:
  \begin{itemize}
  \item The output of $v_1$ is $o_1(x) = f_1(x)$
  \item For $i=2,3,\ldots,k$, the output of $v_i$ is $o_i(x) = f_i( o_j(x) : (j,i) \in E)$
  \end{itemize}
\end{definition}

The bottom-up reasoning process aims at building a computation graph that approximates $g$. It is done by building a search tree (using BFS or DFS), until reaching a root-to-leaf path in the tree that describes a computation graph that approximates $g$. In other words, node $i$ in the root-to-leaf path in the search tree corresponds to vertex $i$ in a computation graph (based on topological order). This is done while relying on an actor and critic policies at each reasoning step, as well as a ``decomposition oracle" for terminating the search process:
\begin{itemize}
\item The \textbf{actor} receives a context $c \in \Sigma^*$ (which represents free text information obtained so far, represented as a sequence of tokens over an alphabet $\Sigma$) as well as a computation graph for the path from the root to the current node. The actor proposes a set of \emph{possibilities} to append one more vertex to the computation graph. 
  Suppose that the current computation graph has $i-1$ vertices. Then each
  possibility is defined by a function $f_i$ and a subset $J \subseteq \{1,\ldots,i-1\}$. This yields a function $o_i$ whose output on $x$ is $o_i(x) = f_i( o_j(x) : j \in J)$. In addition, the actor provides\footnote{Constructing an EV is domain-dependent. For some problems, EV can be constructed by using `nature' as an oracle that calculates correct outputs of $f_i$ (that is, using experimentation). For algorithmic problems, it is often easy to `verify' efficiently but hard to construct a solution (which is the basis of the P vs. NP distinction). For some problems, we can rely on brute force calculations, or on testing of small-size examples and assuming smoothness of the solution with respect to the problem size.} an Example Validator (EV) for each such $f_i$, such that given an input $a$ to $f_i$, the EV outputs True if $f_i(a)$ is correct and False otherwise. In other words, the EV produces a spec for the intended functionality of $f_i$ via a unit test. The critic, discussed in the next item below, is responsible for producing inputs $a$ for the test. We refer to the set of proposals made by the actor as a \emph{proposal class}, $\mathcal{P}$, and each element of $\mathcal{P}$ is a triplet of $(f_i, J, \mathrm{EV})$. It is convenient to think of $\mathcal{P}$  as a finite class, and for simplicity, the analysis we present later is for this case. However, using standard techniques from learning theory, it is not hard to generalize the results to some infinite classes. 
\item The \textbf{critic} receives the proposal class, $\mathcal{P}$, and should output a subset $\mathcal{P}_g \subset \mathcal{P}$ that contains only proposals which are $\epsilon_i$-approximately correct (for a value of $\epsilon_i$ to be specified later on). We say that a proposal is $\epsilon_i$-approximately correct if $\prob_{x \sim \mathcal{D}}[\mathrm{EV}(f_i,a(x)) = \mathrm{False} ] \le \epsilon_i$, where the probability is over an i.i.d. sampling of $x \sim \mathcal{D}$ and $a(x) =  (o_j(x) : j \in J)$ is obtained by running the current computation graph on $x$. We discuss the required value of $\epsilon_i$ as well as how to guarantee that the critic is (probably approximately) correct later on.
\item \textbf{Search Termination by decomposition oracle:} at some step of the search, the actor may think that the computation graph corresponding to the current node in the search tree is representing a correct decomposition of $g$. When that happens, the actor calls it a ``decomposition oracle". The decomposition oracle should approve the computation graph if with probability of at least $1-\epsilon$ over a random sampling of an instance $x \sim \mathcal{D}$, if $\mathrm{EV}(f_i,a_i(x))$ is True for every $i \in [k]$ then $o_k(x) \equiv g(x)$, where '$\equiv$' means that $o_k(x)$ is a correct output to $x$. In other words, the decomposition oracle does not validate the implementation of the individual functions in the computation graph, but merely validates that the decomposition of $g$ to the computation graph (according to its spec) is correct. Similarly to the EV, the method for obtaining a decomposition oracle is domain dependent.
\end{itemize}

The reader familiar with learning theory would notice some similarities between the goal of the critic and the goal of a PAC learner. In both cases, we aim to find functions that agree with some examples up to an $\epsilon$-accuracy. There are, however, some subtle differences. First, in PAC learning, we are aiming to find a single function and not a subset of good functions. Second, in PAC learning all the functions in the hypothesis class are Boolean functions having the same domain, while here, functions in the proposal class may have different domains (i.e., different inputs), different distributions over the domain, and different outputs. 

Despite these differences, as we show in the following lemma,
techniques from learning theory can be generalized to our case.
\begin{lemma} \label{lem:pac_compug}
  Fix $\epsilon,\delta \in (0,1)$. 
  Suppose that $|\mathcal{P}| < \infty$, and let $m \ge \log(|\mathcal{P}|/\delta) / \epsilon$.
  Consider $m$ i.i.d. samples $(x_1,\ldots,x_m)$ from $\mathcal{D}$ and 
 let $\mathcal{P}_g$ be the set of proposals $p = (f, J, \mathrm{EV})$ for which $\mathrm{EV}(f,a(x_t))$ yields True for all $t \in [m]$. Then, with probability of at least $1-\delta$ over the sampling of $(x_1,\ldots,x_m)$, all of the proposals in $\mathcal{P}_g$ are $\epsilon$-accurate.
\end{lemma}
\begin{proof}
  For $p = (f,J,\mathrm{EV}) \in \mathcal{P}$, denote $\epsilon(p) := \prob_{x \sim \mathcal{D}}[\mathrm{EV}(f,a(x)) = \mathrm{False} ]$.
  We need to show that
  \[
\prob[ \exists p \in \mathcal{P}_g ~:~ \epsilon(p) > \epsilon] ~\le~ \delta
\]
where the probability is over the sampling of $(x_1,\ldots,x_m)$.
The left-hand side can be re-written as
\[
\prob[ \exists p=(f, J,\mathrm{EV}) \in \mathcal{P} ~:~ \epsilon(p) > \epsilon~\land~ \mathrm{EV}(f,a(x_1))~\land~\mathrm{EV}(f,a(x_2))~\land~\ldots~\land \mathrm{EV}(f,a(x_m))]
\]
Applying the union bound, this is upper bounded by
\[
\sum_{p=(f, J,\mathrm{EV}) \in \mathcal{P}} \prob[ \epsilon(p) > \epsilon~\land~  \mathrm{EV}(f,a(x_1))~\land~\mathrm{EV}(f,a(x_2))~\land~\ldots~\land \mathrm{EV}(f,a(x_m))]
\]
Each summand in the above sum is upper bounded by $(1-\epsilon)^m$ and therefore the above is upper bounded by
\[
|\mathcal{P}|\,(1-\epsilon)^m \le |\mathcal{P}|\,\exp(-\epsilon\,m) \le \delta ~,
\]
where the last inequality is followed from the definition of $m$.
\end{proof}

Based on the above lemma, we can give guarantees on the entire bottom-up reasoning process.
\begin{theorem}
  Fix $\epsilon,\delta \in (0,1)$ and an integer $ k_{\max}$.
  Let $m$ be an integer that satisfies $m \ge \log(|\mathcal{P}|\,k_{\max}/\delta)\,2\,k_{\max}\, / \epsilon$. 
  Consider a bottom-up reasoning process for learning a function $g : X \to Y$, that ends up with a computation graph of $k \le k_{\max}$ vertices, and suppose that the decomposition oracle is $(\epsilon/2)$-correct. Suppose that the critic relies on $(x_1,\ldots,x_m)$ instances sampled i.i.d. according to a distribution $D$ over $X$.  
  Then, with probability of at least $1-\delta$, the output of the computation graph $\epsilon$-approximates $g$, namely, $\prob_{x \sim D}[o_k(x) \equiv g(x)] \ge 1-\epsilon$. 
\end{theorem}
\begin{proof}
Let $X_{\not d}$ be the set of instances for which the decomposition oracle is correct.
  Based on the assumptions in the theorem, given some $x \in X_{\not d}$, the computation graph might give a wrong answer if for some vertex $i$ we have that $\mathrm{EV}(f_i,a_i(x))$ outputs False. Let $X_i \subseteq X$ denote the set of instances for which $\mathrm{EV}(f_i,a_i(x))$ outputs False. It follows that we need to show that with a probability of at least $1-\delta$ over the sampling of $(x_1,\ldots,x_m)$ we have that $D(\cup_{i \in [k]} X_i) \le \epsilon/2$.

  Define $\hat{\delta} = \delta/k_{\max}$ and $\hat{\epsilon} = \epsilon/(2k_{\max})$. According to \lemref{lem:pac_compug}, with probability of at least $1-\hat{\delta}$, we have that $D(X_i) \le \hat{\epsilon}$. Applying the union bound over every $i \in [k]$ we obtain that with probability of at least $1 - \hat{\delta}\,k$, for each $i \in [k]$, $D(X_i) \le \hat{\epsilon}$. Using again the union bound, this means that
  $D(\cup_{i \in [k]} X_i) \le \hat{\epsilon}\,k$. The theorem follows by using the definitions of $\hat{\delta} = \delta/k_{\max}$ and $\hat{\epsilon} = \epsilon/(2k_{\max})$. 
\end{proof}

\subsection{Top-Down Reasoning}

In \secref{sec:bottomup} we presented a bottom-up reasoning process. We call it bottom-up because a computation graph implements each operation before using it. A different reasoning approach is a top-down process in which we allow ``forward references'' to un-implemented functions.
To motivate top-down reasoning, consider the following stand-alone Python implementation of the merge-sort algorithm.

\begin{verbatim}
def merge_sort(arr):

    if len(arr) <= 1:
        return arr

    current_size = 1

    while current_size < len(arr) - 1:
        for start in range(0, len(arr) - current_size, current_size * 2):
            mid = start + current_size
            end = min(start + 2 * current_size, len(arr))
            arr[start:end] = merge(arr[start:mid], arr[mid:end])

        current_size *= 2

    return arr


def merge(left, right):

    sorted_list = []
    i = j = 0
    while i < len(left) or j < len(right):
        if j >= len(right) or (i < len(left) and left[i] <= right[j]):
            sorted_list.append(left[i])
            i += 1
        else:
            sorted_list.append(right[j])
            j += 1

    return sorted_list

\end{verbatim}

We've intentionally written the code with a top-down structure, starting with the main function we want to implement (``merge\_sort''). The main function relies on a helper function (``merge''). We implement the helper function after completing the main function. 

Attempting to implement merge\_sort using a computation graph, we immediately notice some undesired phenomena. First, while the code itself does not depend on the length of the array we aim to sort, the structure of a computation graph does depend on it. So we need a dedicated computation graph per each length of the array. Furthermore, even if we fix the size of the array, we notice that the ``merge'' function is being called multiple times with different inputs. As a result, we need a dedicated vertex in the computation graph for each call to each function. In other words, we need to re-implement the helper functions again and again. Besides being inefficient, this also unnecessarily enlarges the number of reasoning steps we need to perform. 

This motivates a different approach, which we call a top-down reasoning process. We are again aiming to approximate a function $g : X \to Y$. In the aforementioned example, this function would be ``merge\_sort''.

The top-down reasoning process builds a search tree. Each node of the search tree is labeled by two sets of functions $I, U$, which denote functions that have been already implemented ($I$) and functions that were declared but are still not implemented ($U$, for Unimplemented functions). We initialize the root of the search tree with $I = \emptyset, U = \{g\}$. 

The actor of the top-down reasoning process works on some leaf node of the current search tree. The actor constructs a set of proposals, $\mathcal{P}$, for expanding the leaf, making it an internal node. Each $p \in \mathcal{P}$ proposes an implementation of some unimplemented function, $f \in U$. The proposed implementation may involve new helper functions, which will be added to $I$. So, the child of the node corresponding to this proposal will remove $f$ from $U$ and will add the new helper functions to $I$. Whenever we reach a node for which $U = \emptyset$, we get a complete implementation of the original function $g$ (and the search terminates).

In the bottom-up construction, we also required the actor to provide an Example Validator (EV), that given an input to a node's function and a proposed implementation of the function, the EV assesses the correctness of the output of the proposed function on the given input. Here, we make an additional assumption: The actor should produce a Reference Implementation (RI), potentially a non-efficient one, of each helper function it relies on\footnote{Constructing an RI is domain-dependent. In the context of algorithmic problems, creating RI often amounts to ``translating" the required functionality into code by a brute-force implementation. Evidently, LLMs are fairly proficient at this task.}. This assumption is needed for the critic, as we explain below.

As in the bottom-up reasoning process, the critic receives the proposal class, $\mathcal{P}$, and should output a subset $\mathcal{P}_g \subset \mathcal{P}$ that contains only proposals that are approximately correct. Fix some proposal, and let $f$ be the function it implements. Let $x$ be some input to $f$. Then, we can run $f$ on $x$ by relying on the RI of the helper functions $f$ is using, and then assess its correctness using the EV. But in order to guarantee the correctness of the implementation, we need to test it on inputs that are induced by the original probability over $X$. This was relatively straightforward in the bottom-up process, where we propagated the original input through the computation graph. But here, as evident by the ``merge\_sort'' example, we may call each helper function multiple times with various inputs. We solve this problem by relying on the top-down structure of the process: we require that an implementation of each helper function would come after its usage. By doing so, we can sample $(x_1,\ldots,x_m)$ from the original distribution, $\mathcal{D}$, over the domain of $g$, and for each $x_i$ and each helper function $f_j$, we log all the inputs to calls to $f_j$ we have made while running on $x_i$. We then say that $f_j$ is correct on $x_i$ only if it is correct on \emph{all} of these calls.
This leads to the following analysis.
\begin{lemma} \label{lem:pac_topdown}
  Fix $\epsilon,\delta \in (0,1)$. 
  Suppose that $|\mathcal{P}| < \infty$, and let $m \ge \log(|\mathcal{P}|/\delta) / \epsilon$.
  Consider $m$ i.i.d. samples $(x_1,\ldots,x_m)$ from $\mathcal{D}$.
Fix some $p \in \mathcal{P}$ and let $f$ be the function it implements. 
  For each $x \in X$, let $S(x)$ be the set of inputs to $f$ obtained from all calls to $f$ as a helper function in previous nodes from the root of the search tree when running on $x$. 
  Let $\mathcal{P}_g$ be the set of proposals for which $\mathrm{EV}(f,s)$ yields True for all $s \in S(x_i)$ and all $i \in [m]$. Then, with probability of at least $1-\delta$ over the sampling of $(x_1,\ldots,x_m)$, all of the proposals in $\mathcal{P}_g$ are $\epsilon$-accurate, where we say that $p = (f,\mathrm{EV}, \mathrm{RI})$ is $\epsilon$-accurate if $\epsilon(p) \le \epsilon$ where 
  $\epsilon(p) := \prob_{x \sim \mathcal{D}}[\exists s \in S(x)~:~\mathrm{EV}(f,s) = \mathrm{False} ]$.
\end{lemma}
\begin{proof}
  We need to show that
  \[
\prob[ \exists p \in \mathcal{P}_g ~:~ \epsilon(p) > \epsilon] ~\le~ \delta
\]
where the probability is over the sampling of $(x_1,\ldots,x_m)$.
The left-hand side can be re-written as
\[
  \prob[ \exists p=(f, \mathrm{EV}, \mathrm{RI}) \in \mathcal{P} ~:~
\epsilon(p) > \epsilon~\land~ \forall i \in [m],~\forall s \in S(x_i),~ \mathrm{EV}(f, s)]
\]
Applying the union bound, this is upper bounded by
\[
\sum_{p=(f_i, \mathrm{EV}, \mathrm{RI}) \in \mathcal{P}} \prob[ \epsilon(p) > \epsilon~\land~ \forall i \in [m],~\forall s \in S(x_i),~ \mathrm{EV}(f, s)]
\]
Each summand in the above sum is again upper bounded by $(1-\epsilon)^m$ and therefore the above is upper bounded by
\[
|\mathcal{P}|\,(1-\epsilon)^m \le |\mathcal{P}|\,\exp(-\epsilon\,m) \le \delta ~,
\]
where the last inequality is followed from the definition of $m$.
\end{proof}

A guarantee on the entire top-down reasoning process follows:
\begin{theorem}
  Fix $\epsilon,\delta \in (0,1)$ and an integer $ k_{\max}$.
  Let $m$ be an integer that satisfies $m \ge \log(|\mathcal{P}|\,k_{\max}/\delta)\,k_{\max}\, / \epsilon$. 
  Consider a top-down reasoning process for learning a function $g : X \to Y$, that ends up with a leaf in the search tree with $k \le k_{\max}$ functions. Suppose that the critic relies on $(x_1,\ldots,x_m)$ instances sampled i.i.d. according to a distribution $D$ over $X$. 
  Then, with probability of at least $1-\delta$, the implementation of $g$ by the root of the search tree, and using the helper functions, denoted $\hat{g}$, is $\epsilon$-approximately correct, namely, $\prob_{x \sim D}[\hat{g}(x) \equiv g(x)] \ge 1-\epsilon$. 
\end{theorem}
\begin{proof}
Given some $x \in X$, the output of $\hat{g}(x)$ would not be correct if some call to some helper function is not correct. For some helper function $f$, let $X_f \subseteq X$ denote the set of instances for which there exists $s \in S(x)$ such that $\mathrm{EV}(f,s)$ outputs False. It follows that we need to show that with a probability of at least $1-\delta$ over the sampling of $(x_1,\ldots,x_m)$ we have that $D(\cup_{f \in I} X_f) \le \epsilon$.

  Define $\hat{\delta} = \delta/k_{\max}$ and $\hat{\epsilon} = \epsilon/k_{\max}$. According to \lemref{lem:pac_topdown}, with probability of at least $1-\hat{\delta}$, we have that $D(X_f) \le \hat{\epsilon}$. Applying the union bound over every $f \in I$ we obtain that with probability of at least $1 - \hat{\delta}\,k$, for each $f \in I$, $D(X_f) \le \hat{\epsilon}$. Using again the union bound, this means that
  $D(\cup_{f \in I} X_f) \le \hat{\epsilon}\,k$. The theorem follows by using
 $|I|=k$ and the definitions of $\hat{\delta} = \delta/k_{\max}$ and $\hat{\epsilon} = \epsilon/k_{\max}$. 
\end{proof}

\section{Discussion}

The introduction of Artificial Expert Intelligence (AEI) marks a significant step toward building AI systems that are not only capable of skillful execution but also embody the precision, adaptability, and rigor characteristic of human experts. By integrating the PAC-reasoning framework, AEI addresses two critical challenges in AI: error accumulation in long reasoning chains and the inability to adapt reliably to novel and complex problems. The key innovation lies in the establishment of a ``System 3" reasoning process, which parallels the scientific method, combining empirical validation with structured problem-solving. 
AEI's ability to blend bottom-up and top-down reasoning processes offers a flexible approach to tackling a broad spectrum of problems, from algorithmic challenges to engineering and scientific problem-solving. This adaptability positions AEI as a versatile tool for advancing problem-solving methodologies in domains where existing AI systems struggle to generalize or ensure accuracy.

The reliance on domain-specific decomposition oracles and example validators introduces challenges in scaling the system to highly diverse or poorly understood domains. 
Future research should explore methods to generalize AEI’s principles to broader, less structured domains, potentially by improving its ability to autonomously construct validators and decomposition oracles. 

AEI represents a paradigm shift in AI, emphasizing not just capability but also precision. As AEI systems evolve, they have the potential to redefine our expectations of artificial intelligence, offering a new era of hyper-accelerated innovation driven by expert-level reasoning.

\bibliographystyle{plainnat}
\bibliography{bib}

\end{document}